\documentclass[letterpaper, 10 pt, conference]{ieeeconf}  
\IEEEoverridecommandlockouts                              
\usepackage{amsmath,amssymb} 
\usepackage{mathtools}
\usepackage{bm}
\usepackage{hyperref}
\usepackage{cite}
\usepackage{dsfont}
\usepackage{algorithm} 
\usepackage{algpseudocode}
\usepackage{color}

\allowdisplaybreaks

\usepackage{tikz}
\usetikzlibrary{arrows.meta, positioning, calc}
\usepackage{amsmath, amssymb}
\usepackage{amsthm}

\usepackage{pgfplots}
\pgfplotsset{compat=1.18}
\usepgfplotslibrary{fillbetween}
\usepackage{xcolor}

\definecolor{cP2P}{RGB}{0,90,181}       
\definecolor{cAccel}{RGB}{200,50,50}     
\definecolor{cSAC}{RGB}{0,148,115}       
\definecolor{cPPO}{RGB}{190,120,20}      
\newtheorem{theorem}{Theorem}

\newtheorem{proposition}{Proposition}
\newtheorem{corollary}{Corollary}
\newtheorem{remark}{Remark}

\newtheorem{definition}{Definition}
\newtheorem{problem}{Problem}
\newcommand{\E}{\mathbb{E}}

\newcommand{\piE}{\pi_{\mathrm{P}}}
\newcommand{\piT}{\pi_{\theta}}

\title{\LARGE \bf
PriPG-RL: Privileged Planner-Guided Reinforcement Learning for Partially Observable Systems with Anytime-Feasible MPC
}

\author{Mohsen Amiri$^{1,\dag}$, Mohsen Amiri$^{2,\dag}$, Ali Beikmohammadi$^1$, Sindri Magnu\'sson$^1$,\\ and Mehdi Hosseinzadeh$^2$, \IEEEmembership{Senior Member, IEEE}
\thanks{This work was supported by the United States National Science Foundation (awards ECCS-2515358 and CNS-2502856), the Swedish Research Council (grant 2024-04058), and Sweden’s Innovation Agency (Vinnova). Computational resources were provided by the National Academic Infrastructure for Supercomputing in Sweden (NAISS) at C3SE, partially funded by the Swedish Research Council (grant 2022-06725).}
\thanks{$^\dag$ The authors contributed equally to this work.}
\thanks{$^1$The authors are with the Department of Computer and
System Science, Stockholm University, 11419 Stockholm, Sweden, (Email: mohsen.amiri@dsv.su.se).}
\thanks{$^2$The authors are with the School of Mechanical and Material Engineering, Washington State University, Pullman, WA 99164, USA (Email: \{mohsen.amiri, mehdi.hosseinzadeh\}@wsu.edu. }
}

\begin{document}

\maketitle
\thispagestyle{empty}
\pagestyle{empty}

\begin{abstract}
This paper addresses the problem of training a reinforcement learning (RL) policy under partial observability by exploiting a privileged, anytime-feasible planner agent available exclusively during training. We formalize this as a Partially Observable Markov Decision Process (POMDP) in which a planner agent with access to an approximate dynamical model and privileged state information guides a learning agent that observes only a lossy projection of the true state. To realize this framework, we introduce an anytime-feasible Model Predictive Control (MPC) algorithm that serves as the planner agent. For the learning agent, we propose Planner-to-Policy Soft Actor-Critic (P2P-SAC), a method that distills the planner agent's privileged knowledge to mitigate partial observability and thereby improve both sample efficiency and final policy performance. We support this framework with rigorous theoretical analysis. Finally, we validate our approach in simulation using NVIDIA Isaac Lab and successfully deploy it on a real-world Unitree Go2 quadruped navigating complex, obstacle-rich environments.
\end{abstract}

\section{Introduction}

Model-free deep Reinforcement Learning (RL) can produce policies that execute with very low latency on resource-constrained platforms~\cite{beikmohammadi2023ta,mnih2015human,duan2016benchmarking}. However, a fundamental challenge arises when the learning agent has only partial access to the true environment state. In this Partially Observable Markov Decision Process (POMDP) setting, observations do not fully determine the latent state, severely destabilizing value functions conditioned solely on observations~\cite {lauri2022partially}. Consequently, standard RL methods like PPO~\cite{schulman2017proximal}, TD3~\cite{fujimoto2018addressing}, and Soft Actor--Critic (SAC)~\cite{haarnoja2018soft} frequently fail. State aliasing yields uninformative early exploration~\cite{beikmohammadi2024accelerating}, trapping policies in suboptimal local minima~\cite{janner2019trust,amodei2016concrete,uesato2018rigorous} and making convergence prohibitively slow~\cite{oh2025discovering,shakya2023reinforcement,sutton1998reinforcement}. A line of approaches to this challenge is to optimize reactive (memoryless) policies within a surrogate MDP induced by the observation space, accepting an inherent approximation gap~\cite{singh1994learning}. In a specific class of problems known as SNS-MDPs, where the unobservable components follow an autonomous Markov chain, fundamental proofs demonstrate that conventional RL algorithms can safely converge to an average MDP corresponding to the hidden states' stationary distribution~\cite{amiri2024convergence, amiri2025reinforcement}. However, in general continuous control, the surrogate MDP is highly policy-dependent and riddled with state aliasing. 

\begin{figure}
    \centering
    \includegraphics[width=0.95\linewidth]{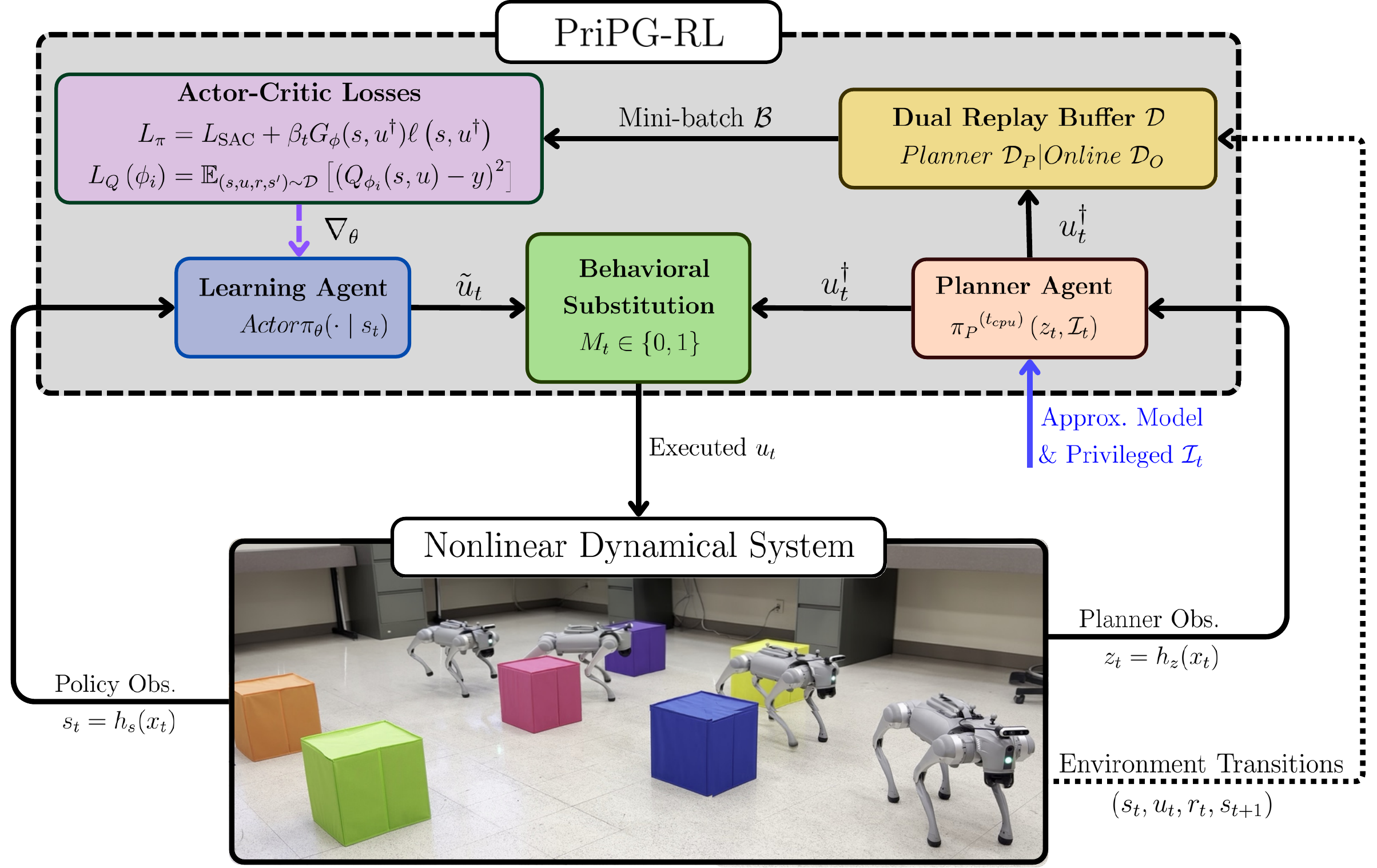}\vspace{-0.2cm}
    \caption[Illustration of the proposed PriPG-RL architecture.]{\footnotesize Illustration of the proposed PriPG-RL architecture during training. The planner agent provides guidance exclusively during training and is not used at runtime. The hardware image is included for visualization purposes only and does not represent a closed-loop deployment of the training architecture. A video demonstration of the hardware deployment, along with the complete source code, is available at GitHub.\footnote{\scriptsize \url{https://github.com/mohsen1amiri/PriPG-RL_UnitreeGo2.git}}}\vspace{-0.5cm}
    \label{fig:PriPG-RL architecture}
\end{figure}
\footnotetext{\scriptsize \url{https://github.com/mohsen1amiri/PriPG-RL_UnitreeGo2.git}}

To bridge this optimality gap, a natural remedy is privileged learning, where a teacher with full state access guides a student operating under restricted observations~\cite{chen2020learning,lee2020learning,margolis2022rapid,kumar2021rma}. In parallel, the RL community has developed robust methodologies to incorporate prior knowledge into training. Under the umbrella of RL from demonstrations (RLfD), methods like DQfD~\cite{hester2018deep}, DDPGfD~\cite{vecerik2018leveraging}, and AWAC~\cite{nair2020awac} augment learning with expert data. Alternatively, model-based planning can generate online training targets~\cite{lowrey2018plan,nagabandi2018neural}. Specifically, recent work~\cite{beikmohammadi2024accelerating} showed that regularizing an SAC actor toward an approximate policy of a heuristic algorithm via a quadratic pseudo-label loss accelerates learning. However, a critical limitation of these RLfD and regularization frameworks is that they are mathematically formulated for fully observable MDPs. Consequently, they struggle to resolve the POMDP context. For instance, the output-space imitation in~\cite{beikmohammadi2024accelerating} assumes a one-to-one state-action mapping, which suffers from vanishing gradients at the SAC actor's boundaries, disproportionately paralyzing the network during safety-critical evasive maneuvers in aliased states. Furthermore, it employs a linear decay schedule that eventually eliminates the heuristic algorithm's guidance entirely. Because the underlying problem is a POMDP, once this guidance decays to zero, the agent is thrust back into an unmitigated environment with severe state aliasing, causing catastrophic forgetting of the safe approximate policy.

Separately, the control community has developed anytime optimization methods that guarantee feasible solutions at any point during computation. The anytime-feasible MPC framework, referred to as REAP~\cite{Hosseinzadeh2023RobustTermination,amiri2025reap,amiri2025practical}, provides such guarantees through a modified barrier function and a primal–dual gradient flow, with solution quality improving monotonically as additional computation is allocated. In contrast, standard MPC offers no feasibility guarantees if terminated before solver convergence, making it unsuitable for online training under varying computational budgets~\cite{ren2022tutorial}. The anytime-feasibility property of REAP makes it a natural candidate for providing structured guidance to RL agents.

This paper proposes a general framework for planner-guided actor--critic RL under partial observability, called Privileged Planner-Guided RL (PriPG-RL); see Figure \ref{fig:PriPG-RL architecture}. The framework is defined by two agents with asymmetric information: i)~a \textit{`planner agent'} with access to an approximate dynamical model and privileged information, and ii)~a \textit{`learning agent'} that observes only a lossy projection of the true state and must function autonomously at deployment. The framework formally characterizes the informational asymmetry and provides mechanisms for the learning agent to extract behavioral priors from the planner agent, performing privileged information distillation, while ensuring the learned policy is not bounded by the planner agent's performance. We make two instantiations. As the planner agent, we develop a REAP-based framework that provides always-feasible guidance at controllable computational cost. As the learning agent, we propose Planner-to-Policy Soft Actor-Critic (P2P-SAC), which leverages the planner agent's signals and improves sample efficiency through four mechanisms. The system is validated in NVIDIA Isaac Lab and deployed on a Unitree Go2 quadruped navigating complex, obstacle-rich environments.

\section{Preliminaries and Problem Statement}\label{sec:ProblemStatement}

\subsection{Dynamical System and Observation Model}\label{sec:DynSys}
Consider the discrete-time dynamical system
\begin{align}\label{eq:TrueSystem}
    x_{t+1} = f(x_t, u_t), \quad t \in \mathbb{Z}_{\geq 0},
\end{align}
where $x_t \in \mathcal{X} \subseteq \mathbb{R}^n$ is the full state,
$u_t \in \mathcal{U} \subseteq \mathbb{R}^p$ is the control input, and
$f : \mathcal{X} \times \mathcal{U} \to \mathcal{X}$ is continuous but unknown dynamic of system. The full state $x_t$ may not be directly accessible to all agents. Let
$\mathcal{H} := \{ h : \mathcal{X} \to \mathcal{Y} \mid \mathcal{Y} \subseteq \mathbb{R}^{n_h},\; n_h \leq n \}$
denote the set of measurement maps. The learning agent observation map $h_s : \mathcal{X} \to \mathcal{S}$, $s_t = h_s(x_t) \in \mathcal{S} \subseteq \mathbb{R}^{n_s}$,
produces the observable state available to the learning agent; when $n_s < n$, the map is information-lossy, referred to as informational incompleteness. The planner agent observation map
$h_z : \mathcal{X} \to \mathcal{Z}$, $z_t = h_z(x_t) \in \mathcal{Z} \subseteq \mathbb{R}^{n_z}$,
produces the planner agent's state. In general, $h_s \neq h_z$, reflecting informational requirements.

 
\subsection{Partially Observable Markov Decision Process (POMDP)}\label{sec:MDP}
 
The closed-loop interaction of a policy with~\eqref{eq:TrueSystem}
under the lossy observation map $h_s$ induces a
POMDP
\begin{align}\label{eq:POMDP}
    \mathcal{P} = (\mathcal{X},\, \mathcal{U},\, f,\,
                   \mathcal{S},\, h_s,\, r,\, \gamma),
\end{align}
where $\mathcal{X}$ is the (hidden) state space,
$f$ is as in~\eqref{eq:TrueSystem},
$\mathcal{S}$ is the observation space,
$h_s : \mathcal{X} \to \mathcal{S}$ is the (deterministic) emission map,
$r : \mathcal{X} \times \mathcal{U} \to \mathbb{R}$ is a bounded reward,
and $\gamma \in [0,1)$ is the discount factor.
When $h_s$ is not injective ($n_s < n$), the observation $s_t$
does not determine the latent state $x_t$, and the process
$\{s_t\}_{t \geq 0}$ is not Markov in general.

Optimizing $\mathcal{P}$ is computationally intractable because it requires history-dependent policies or belief states. However, consider standard RL employs reactive, memoryless policies $\pi_\theta : \mathcal{S} \to \Delta(\mathcal{U})$. This approach induces a surrogate MDP $\widehat{\mathcal{M}} = (\mathcal{S}, \mathcal{U}, \widehat{P}, r, \gamma)$, where the transition kernel marginalizes true dynamics over unobservable latent states $x \in \mathcal{X}$ via the stationary conditional distribution $b_\pi(x \mid s)$:
\begin{align}\label{eq:SurrogateKernel}
    \widehat{P}(s' \mid s, u) = \int_{\mathcal{X}} \mathds{1}\!\big[h_s(f(x, u)) = s'\big]\; d\,b_\pi(x \mid s).
\end{align}
Because $b_\pi$ is shaped by $\pi$, the surrogate kernel $\widehat{P}$ is policy-dependent. As $\pi_\theta$ updates, the resulting non-stationarity in $\widehat{\mathcal{M}}$ violates standard Bellman convergence. Furthermore, state aliasing introduces irreducible epistemic variance into temporal difference targets, often leading to instability in critic function estimation and policy gradient collapse in continuous control POMDPs. Convergence is only theoretically guaranteed if latent states evolve independently, reducing the system to a stationary average MDP \cite{amiri2024convergence, amiri2025reinforcement}. That is why, in general, conventional RL algorithms such as SAC, PPO, and TD3 do not work properly in this case.

\subsection{Linear Approximate Model}\label{sec:LinearModel}
Although $f$ is unknown, a stabilizable LTI approximation is assumed available for planning on $z_t = h_z(x_t)$:
\begin{align}\label{eq:LinearModel}
    z_{t+1} = Az_t + Bu_t,
\end{align}
where $A \in \mathbb{R}^{n_z \times n_z}$, $B \in \mathbb{R}^{n_z \times p}$. This model may arise from linearization, system identification, or physics-based modeling. To robustify feasibility against the modeling residual $\delta_f(x_t, u_t) := h_z(f(x_t, u_t)) - (Az_t + Bu_t)$, the planner agent operates on tightened convex inner-approximations:
\begin{align}\label{eq:TightenedSets}
    \tilde{\mathcal{Z}} &= \{z \in \mathbb{R}^{n_z} : a_i^\top z + b_i \leq 0,\;
    i=1,\ldots,c_x\} \subseteq \mathcal{X}, \\
    \tilde{\mathcal{U}} &= \{u \in \mathbb{R}^p : c_i^\top u + d_i \leq 0,\;
    i=1,\ldots,c_u\} \subseteq \mathcal{U}.
\end{align}

\subsection{Model Predictive Control}\label{sec:MPC}
Let $d \in \mathbb{R}^m$ be a desired reference with steady-state configuration $(\bar{z}_d, \bar{u}_d)$ satisfying $\bar{z}_d = A\bar{z}_d + B\bar{u}_d$, $\bar{z}_d \in \operatorname{Int}(\tilde{\mathcal{Z}})$, $\bar{u}_d \in \operatorname{Int}(\tilde{\mathcal{U}})$. Given prediction horizon $N \in \mathbb{Z}_{>0}$, MPC computes the optimal control sequence $\mathbf{u}^\ast_t \in \mathbb{R}^{pN}$ as\begin{subequations}\label{eq:OptimizationProblemMain}
\begin{align}
\mathbf{u}^\ast_t := \min_{\mathbf{u}} \;& \sum_{\kappa=0}^{N-1}\left\|\hat{z}_{\kappa \mid t} - \bar{z}_d\right\|_{{\mathcal{L}_x}}^2 + \left\|u_{\kappa \mid t} - \bar{u}_d\right\|_{{\mathcal{L}_u}}^2 \nonumber \\
& + \left\|\hat{z}_{N \mid t} - \bar{z}_d\right\|_{{\mathcal{L}_N}}^2\label{eq:MPCCostFunction}
\end{align}subject to\begin{align}
& \hat{z}_{\kappa+1 \mid t} = A \hat{z}_{\kappa \mid t} + B u_{\kappa \mid t},\; \hat{z}_{0 \mid t} = z_t \label{eq:modelConstratins} \\
& \hat{z}_{\kappa \mid t} \in \tilde{\mathcal{Z}},\; u_{\kappa \mid t} \in \tilde{\mathcal{U}},\; \kappa \in \{0,\ldots,N{-}1\} \label{eq:Xconstraints} \\
& (\hat{z}_{N \mid t}, d) \in \Omega, \label{eq:Tconstraints}
\end{align}
\end{subequations}
where $\mathcal{L}_x \succeq 0$, $\mathcal{L}_u \succ 0$, $\mathcal{L}_N \succ 0$ are weighting matrices and $\Omega$ is the terminal constraint set. The computation of $\mathcal{L}_N$ and $\Omega$ is addressed in Section~3.6 of \cite{amiri2025reap}.

Constraints \eqref{eq:modelConstratins}--\eqref{eq:Tconstraints} can be rewritten as constraints on the control sequence $\mathbf{u}$ by recursively substituting the system dynamics~\eqref{eq:LinearModel}. This results in the compact polyhedral set
\begin{align}\label{eq:Constraintall}
\mathbb{U}=
\left\{
\mathbf{u}\in\mathbb{R}^{pN} :
\eta_i^\top\mathbf{u}+g_i\le0,\;
i=1,\ldots,\bar{c}
\right\},
\end{align}
where $\eta_i\in\mathbb{R}^{pN}$ and $g_i\in\mathbb{R}$ are constants obtained from the state, input, and terminal constraints, and $\bar{c}$ denotes the total number of resulting constraints. We introduce the set $\mathfrak{U}=\mathrm{Proj}_p(\mathbb{U})$ where $\mathrm{Proj}_p(\cdot)$ extracts the first $p$ elements.
 
\subsection{Soft Actor--Critic}\label{sec:SAC}
SAC~\cite{haarnoja2018soft} is a model-free, off-policy maximum-entropy algorithm originally designed for fully observable MDPs. When applied to the POMDP setting with reactive policies $\pi_\theta : \mathcal{S} \to \Delta(\mathcal{U})$, the critic functions $Q_{\phi_j}(s,u)$ are conditioned on observations rather than full states, facing the instabilities described in Section~\ref{sec:MDP}. SAC maximizes
\begin{align}\label{eq:RLObjective}
    J(\pi_\theta) = \mathbb{E}_{\tau \sim \pi_\theta}\!\left[
        \textstyle\sum_{t=0}^{\infty}\gamma^t\Big(
            r_t + \alpha\,\mathcal{H}(\pi_\theta(\cdot|s_t))
        \Big)
    \right],
\end{align}
where $\alpha \geq 0$ is the entropy temperature. SAC maintains twin critics $Q_{\phi_i}$, $i \in \{1,2\}$, minimizing the soft Bellman residual $L_Q(\phi_i) = \mathbb{E}_{(s,u,r,s') \sim \mathcal{D}}[(Q_{\phi_i}(s,u) - y)^2]$ with target $y = r + \gamma(\min_j Q_{\phi_{j,\mathrm{targ}}}(s',\tilde{u}') - \alpha\log\pi_\theta(\tilde{u}'|s'))$, $\tilde{u}' \sim \pi_\theta(\cdot|s')$. The actor minimizes
\begin{equation}\label{eq:SACactor}
    L_\pi(\theta) = \mathbb{E}_{s \sim \mathcal{D},\,
    \tilde{u} \sim \pi_\theta}\!\left[
        \alpha\log\pi_\theta(\tilde{u}|s)
        - \min_{j}Q_{\phi_j}(s,\tilde{u})
    \right].
\end{equation}
Despite these limitations in the POMDP setting, SAC provides a principled actor--critic foundation. The P2P-SAC algorithm proposed in Section~\ref{sec:P2P-SAC} builds on this foundation by incorporating planner agent guidance to overcome the challenges of partial observability.

\subsection{Problem Formulation}\label{sec:Formulation}

We now formalize the information structures and state the main problem.

\begin{definition}[Privileged Information]\label{def:PrivInfo}
    The privileged information set $\mathcal{I}_t$ is any task-relevant information available to the planner agent beyond $s_t$, such that $s_t$ is a deterministic function of $\mathcal{I}_t$. Typical elements include the full state $x_t$, constraint geometry, and environment parameters.
\end{definition}

\begin{definition}[Anytime-Feasible Planner Agent]\label{def:AnytimePlanner}
    An anytime-feasible planner agent is a deterministic mapping $\piE^{(t_{cpu})} : \mathcal{Z} \times \mathcal{I} \to \mathfrak{U}~$, parameterized by computation time $t_{cpu} \geq 0$, producing $u^{\dagger}_t = \piE^{(t_{cpu})}(z_t, \mathcal{I}_t)$. The policy strictly preserves feasibility, ensuring $\piE^{(t_{cpu})}(z_t, \mathcal{I}_t) \in \mathfrak{U}$ regardless of when computation terminates, and is asymptotically optimal, satisfying $\piE^{(t_{cpu})}(z_t, \mathcal{I}_t) \to u^\ast_t$ as $t_{cpu} \to \infty$, where $u^\ast_t$ is the first element of the sequence $\mathbf{u}_t^\ast$ defined in \eqref{eq:OptimizationProblemMain}. The suboptimality gap is defined as $\Delta^{t_{cpu}} := \|\piE^{(t_{cpu})}(z_t, \mathcal{I}_t) - u^\ast_t\|$.
\end{definition}


The planner agent $\piE^{(t_{cpu})}$ operates on $(z_t, \mathcal{I}_t)$ using~\eqref{eq:LinearModel}, while the learned policy $\pi_\theta$ operates exclusively on $s_t$ with no access to $z_t$, $\mathcal{I}_t$, or $(A,B)$ at deployment. This informational asymmetry is formalized below.

\begin{definition}[Informational Asymmetry Gap]\label{def:InfoGap}
    The informational asymmetry gap is defined as
    \begin{align}
        \mathcal{G} := \bigl\{
            s \in \mathcal{S}
            \;\big|\;
            &\exists\, x, x' \in h_s^{-1}(s),\;
            \exists\, \mathcal{I}, \mathcal{I}' \in \mathfrak{I} :
            \notag\\
            &\piE(h_z(x), \mathcal{I})
            \neq
            \piE(h_z(x'), \mathcal{I}')
        \bigr\},
    \end{align}
    where $h_s^{-1}(s) := \{x \in \mathcal{X} \mid h_s(x) = s\}$. States in $\mathcal{G}$ exhibit \emph{state aliasing}: identical observations $s$ map to distinct latent states and planner agent's actions.
\end{definition}

\begin{remark}[Privileged Information Distillation]\label{rem:distillation}
    The planner agent's action $u_t = \piE(z_t, \mathcal{I}_t)$ relies on privileged information strictly subsuming the learning agent's observation $s_t$. However, the learning agent distills this richer information to partially mitigate the partial observability. 
\end{remark}

\begin{problem}\label{prob:Main}
Given $\mathcal{P}$ as in \eqref{eq:POMDP}, the privileged information set $\mathcal{I}_t$, and an anytime-feasible planner agent $\piE^{(t_{cpu})}$ (Definition~\ref{def:AnytimePlanner}), find a reactive policy $\pi_{\theta^*} : \mathcal{S} \to \Delta(\mathcal{U})$ satisfying: (1) reactive optimality, such that $\pi_{\theta^*} = \arg\max_\theta J(\pi_\theta)$ among all reactive policies $\pi_\theta : \mathcal{S} \to \Delta(\mathcal{U})$; (2) deployment autonomy, where $\pi_{\theta^*}$ utilizes only $s_t$ at execution time without access to $\mathcal{I}_t$ and $\piE^{(t_{cpu})}$; and (3) training efficiency, whereby the sample complexity to achieve $J(\pi_{\theta^*}) - J(\pi_\theta) \leq \varepsilon$ is reduced by exploiting $\piE^{(t_{cpu})}$ during training.
\end{problem}
 
\begin{remark}
    Objective~(1) targets the best reactive policy,
    not the POMDP-optimal history-dependent policy.
    Since reactive policies cannot resolve state aliasing in $\mathcal{G}$, there is an inherent optimality gap relative to $\mathcal{P}$.
    Objective~(3) motivates using $\piE^{(t_{cpu})}$ as a training signal. A central challenge is that na\"{i}ve imitation may prevent the learned policy from surpassing the planner agent when $\mathcal{G} \neq \emptyset$, since the planner agent's actions in $\mathcal{G}$ depend on privileged information that no reactive policy on $\mathcal{S}$ can replicate.
    The proposed framework addresses this through the mechanisms in Section~\ref{sec:P2P-SAC}.
\end{remark}

\section{Anytime-Feasible MPC\\ (REAP-Based Planner Agent)}\label{sec:AMPC}

Inspiring from \cite{Hosseinzadeh2023RobustTermination,amiri2025reap}, we develop an anytime-feasible MPC-based planner agent $\piE^{(t_{cpu})}$ parameterized by a \emph{computation time} $t_{cpu} \geq 0$ to be used in the PriPG-RL framework, which will be introduced in the next section. To do so, we introduce our method for solving problem~\eqref{eq:OptimizationProblemMain} in real time. Consider the barrier function:
\begin{align}\label{eq:BarrierFunction}
\mathcal{B}(\mathbf{z}_t,d,\mathbf{u},\lambda)
=&J(\mathbf{z}_t,\mathbf{u},d)\\
&-\sum_{i=1}^{\bar{c}}
\lambda_i
\log\!\left(
-\beta(\eta_i^\top\mathbf{u}+g_i+\tfrac{1}{\omega})+1
\right),
\nonumber
\end{align}
where $J(\cdot)$ is the cost function defined in \eqref{eq:MPCCostFunction}, $\lambda=[\lambda_1,\cdots,\lambda_{\bar{c}}]^\top\in\mathbb{R}^{\bar{c}}$ is the vector of dual variables, and $\omega \in \mathbb{R}_{>0}$ is a tightening parameter chosen sufficiently large to avoid excessive conservatism, as discussed in~\cite{amiri2026dynamic}. The barrier function in \eqref{eq:BarrierFunction} is strongly convex in $\mathbf{u}$, since $J(\cdot)$ is strongly convex in $\mathbf{u}$, and therefore admits a unique global minimizer, denoted by $\mathbf{u}_t^\dagger$. Moreover, $\mathbf{u}_t^\dagger \to \mathbf{u}_t^\ast$ as $\omega \to \infty$, where $\mathbf{u}_t^\ast$ is the optimizer of \eqref{eq:OptimizationProblemMain}. The corresponding optimal dual variables at time $t$ are denoted by $\lambda_t^\dagger$. We reasonably assume that the linear independence constraint qualification holds at $\mathbf{u}_t^\dagger$, which implies that $\lambda_t^\dagger$ is unique.

The anytime-feasible MPC-based planner agent $\piE^{(t_{cpu})}$ evolves the following virtual continuous-time dynamical system based on a primal--dual gradient flow\begin{subequations}\label{eq:REAP}
\begin{align}
\frac{d}{d\rho}\hat{\mathbf{u}}_\rho
&=-\zeta
\nabla_{\hat{\mathbf{u}}}
\mathcal{B}\big(\mathbf{z}_t,d,\hat{\mathbf{u}}_\rho,\hat{\lambda}_\rho\big),
\\
\frac{d}{d\rho}\hat{\lambda}_\rho
&=\zeta\Big(
\nabla_{\hat{\lambda}}
\mathcal{B}\big(\mathbf{z}_t,d,\hat{\mathbf{u}}_\rho,\hat{\lambda}_\rho\big)
+\Psi_\rho
\Big),
\end{align}
\end{subequations}
where $\rho$ denotes the computation time within the sampling period $t$ and $\zeta>0$ is a design parameter determining the evolution speed of \eqref{eq:REAP}. The function $\Psi_\rho:\mathbb{R}^{\bar{c}} \to \mathbb{R}^{\bar{c}}$ is a projection operator that ensures the non-negativity of the dual variables $\lambda_i,~\forall i$; see \cite{Hosseinzadeh2023RobustTermination} for details.

Following \cite{Hosseinzadeh2023RobustTermination}, the trajectories of the dynamical system satisfy the following properties. Proofs are omitted due to space limitations and follow directly from the same steps.

\begin{proposition}
\label{theorem:Convergence}
Let $(\hat{\mathbf{u}}_\rho, \hat{\lambda}_\rho)$ be the trajectory of \eqref{eq:REAP}. Given a feasible initial condition $(\hat{\mathbf{u}}_0, \hat{\lambda}_0),~(\hat{\mathbf{u}}_\rho, \hat{\lambda}_\rho)$ exponentially converges to $\left(\mathbf{u}^{\dagger}_t, \lambda^{\dagger}_t\right)$ as $\rho \rightarrow \infty$.
\end{proposition}


\begin{proposition}
\label{theorem:Feasibility}
Let $(\hat{\mathbf{u}}_\rho, \hat{\lambda}_\rho)$ be the solution of \eqref{eq:REAP}. Given a feasible initial condition $(\hat{\mathbf{u}}_0, \hat{\lambda}_0)$, $\hat{\mathbf{u}}_\rho$ satisfies constraints  \eqref{eq:Constraintall}  for all $\rho$, i.e., $\hat{\mathbf{u}}_\rho\in\mathbb{U}$ for all $\rho$.
\end{proposition}

Regarding Propositions \ref{theorem:Convergence} and \ref{theorem:Feasibility}, the dynamical system \eqref{eq:REAP} ensures the resulting solution is feasible, yet suboptimal, while allowing for adjustable computational time $t_{cpu}$. Consequently, the balance between suboptimality and speed can be tuned, offering an adaptable solution for use in the PriPG-RL framework as the planner agent. This adaptability allows \eqref{eq:REAP} to effectively handle limited and varying computational resources, while maintaining feasibility and achieving control objectives. These properties could potentially help the RL to reduce optimality in early training, guide early exploration, and enhance learning efficiency. Leveraging the anytime feasibility guaranteed by Proposition \ref{theorem:Feasibility}, the evolution of the continuous-time dynamical system \eqref{eq:REAP} can be safely terminated at any point, typically dictated by a pre-defined computational time budget $t_{cpu}$. The first element of the resulting control sequence  $\hat{\mathbf{u}}_\rho$ at termination is then extracted and utilized as the planner agent $\piE^{(t_{cpu})}$. 


\section{P2P-SAC: Planner-to-Policy Soft Actor-Critic Reinforcement Learning}\label{sec:P2P-SAC}

We propose the P2P-SAC algorithm as a specific instantiation of $\piT : \mathcal{S} \to \Delta(\mathcal{U})$ that addresses all three objectives in Problem~\ref{prob:Main}. The REAP-based framework of Section~\ref{sec:AMPC} serves as the planner agent $\piE^{(t_{cpu})}$ (Definition~\ref{def:AnytimePlanner}), producing ${u}^{\dagger}_t \in \mathfrak{U}$ at any budget $t_{cpu}$ using $(z_t, \mathcal{I}_t)$ and available only during training. The learned policy $\piT$ operates exclusively on $s_t = h_s(x_t)$ and requires no access to $z_t$, $\mathcal{I}_t$, or $(A,B)$ at deployment.

P2P-SAC couples four mechanisms to exploit $\piE^{(t_{cpu})}$ without bounding the asymptotic performance of $\piT$: (i)~a dual replay buffer, (ii)~a deterministic three-phase maturity schedule, (iii)~a logit-space imitation anchor, and (iv)~an advantage-based sigmoid gate.

\subsection{Dual Replay Buffer}\label{sec:DualBuffer}

Unlike prior RLfD methods that rely on fixed, pre-collected demonstrations~\cite{hester2018deep,vecerik2018leveraging,nair2018overcoming}, P2P-SAC generates its planner buffer online via behavioral substitution, ensuring it reflects the closed-loop dynamics of~\eqref{eq:TrueSystem} under $\piE^{(t_{cpu})}$.

\begin{definition}[Dual Replay Buffer]\label{def:DualBuffer}
    Given capacities $C_P < C$, the dual replay buffer $\mathcal{D} = (\mathcal{D}_P, \mathcal{D}_O)$ comprises: (1)~a write-once planner agent's buffer $\mathcal{D}_P$ of capacity $C_P$ that freezes transitions collected when $M_t=0$ (Subsection~\ref{sec:AdvantageGate}), and (2)~a standard FIFO online buffer $\mathcal{D}_O$ of capacity $C - C_P$. Each stored transition is an augmented tuple $(s_t, u_t, r_t, s_{t+1}, {u}^{\dagger}_t, h_t)$, where $h_t \in \{0,1\}$ indicates whether a valid planner agent's action was queried.
\end{definition}

During the immature phase ($M_t = 0$, Definition~\ref{def:Schedule}), the planner agent's action replaces the executed input via:
\begin{align}\label{eq:BehaviorSubstitution}
    u_t =
    \begin{cases}
        {u}^{\dagger}_t, & \text{if } M_t = 0 \text{ and } h_t = 1, \\
        \tilde{u}_t, & \text{otherwise},
    \end{cases}
\end{align}
where $\tilde{u}_t \sim \piT(\cdot \mid s_t)$. At each gradient step, a mini-batch $\mathcal{B}$ is assembled as an equal-weight mixture:
\begin{align}\label{eq:MixedSampling}
    \mathcal{B} = \mathcal{B}_P \cup \mathcal{B}_O,
    \quad
    &\mathcal{B}_P \sim \mathrm{Uniform}(\mathcal{D}_P,\, B_P), \notag \\
    & \mathcal{B}_O \sim \mathrm{Uniform}(\mathcal{D}_O,\, B_O),
\end{align}
with $B_P = \lfloor B/2 \rfloor$, drawing entirely from the non-empty buffer if either is empty.

\subsection{Deterministic Three-Phase Maturity Schedule}\label{sec:schedule}

In the recent work~\cite {beikmohammadi2024accelerating}, annealing schedules decay the guidance weight to zero, extinguishing the signal regardless of whether the critic function is reliable or the planner agent remains superior. Once this guidance is completely removed, the method reverts to standard RL, where the restricted observation space fails to form a valid MDP, exposing the agent to the unmitigated state aliasing of the underlying POMDP. P2P-SAC instead employs a deterministic schedule parameterized by plateau horizon $T_p$, annealing horizon $T_d$, and guidance coefficients $\beta_0, \beta_f > 0$.

\begin{definition}[Three-Phase Maturity Schedule]\label{def:Schedule}
    The guidance positive coefficient $\beta_t \in [\beta_f, \beta_0]$ and maturity indicator $M_t \in \{0, 1\}$ evolve sequentially. During the plateau phase ($0 \leq t \leq T_p$), the agent is immature ($M_t = 0$) with $\beta_t = \beta_0$, keeping behavioral substitution~\eqref{eq:BehaviorSubstitution} active and routing transitions to $\mathcal{D}_P$. During the annealing phase ($T_p < t \leq T_p + T_d$), $M_t$ remains $0$ while the coefficient decays via $\beta_t = \beta_0 - \tfrac{t - T_p}{T_d}(\beta_0 - \beta_f)$. Finally, in the maturity phase ($t > T_p + T_d$), $\beta_t = \beta_f$ and $M_t = 1$. 
\end{definition}

This absorbing maturity state avoids cyclic deadlocks, deactivates substitution, grants $\piT$ full autonomy, and routes transitions to $\mathcal{D}_O$. By maintaining a non-zero final guidance positive coefficient $\beta_f$ alongside the advantage gate, P2P-SAC prevents the catastrophic return to unmitigated partial observability, ensuring the agent remains robust to state aliasing even after reaching maturity.

\subsection{Logit-Space Imitation Anchor}\label{sec:LogitAnchor}

The learning agent $\piT$ uses a squashed Gaussian actor: $\tilde{u} = \tanh(\mu_\theta(s) + \epsilon)$, where $\mu_\theta(s) \in \mathbb{R}^p$ is the pre-activation mean (logit) and $\epsilon \sim \mathcal{N}(0, \sigma_\theta(s)^2 I)$. Any imitation loss on the squashed output $\tilde{u}$ has gradient scaled by $(1 - \tanh^2(\mu_\theta(s)))$, which vanishes exponentially as $\|\mu_\theta(s)\|_\infty$ grows. Since planner agent's actions near $\partial\mathfrak{U}$ correspond to large logits, output-space losses~\cite{beikmohammadi2024accelerating,hester2018deep,vecerik2018leveraging} produce near-zero gradients at safety-critical operating points. However, P2P-SAC resolves this by anchoring in logit space. Given $u^\dagger = \piE^{(t_{cpu})}(z_t, \mathcal{I}_t) \in \mathfrak{U}$ with bounds $[u_{\mathrm{low}}, u_{\mathrm{high}}]^p$:

\begin{definition}[Planner Agent's Logit]\label{def:PlannerLogit}
    For a planner agent's action $u^\dagger \in \mathfrak{U}$ and numerical margin $\varepsilon \in (0, 1)$, the planner agent's logit $\xi^\dagger$ is derived via the following component-wise operations:
    \begin{align}\label{eq:PlannerLogit}
        \xi^\dagger \!\!= \mathrm{tanh^{-1}}\!\!\left(\mathrm{clip}\!\left(\left(2\,\frac{u^\dagger - u_{\mathrm{low}}}{u_{\mathrm{high}} - u_{\mathrm{low}}} \!-\! \mathbf{1}\right),\; \!\!-L,\;\!\! L\right)\right) \in \mathbb{R}^p, 
    \end{align}
    where $L=(1{-}\varepsilon)$, and $\varepsilon$ ensures $\xi^\dagger$ remains finite near the boundary $\partial\mathfrak{U}$.
\end{definition}

The per-sample logit-space imitation loss is
\begin{align}\label{eq:ImitationLoss}
    \ell(s,\, u^\dagger_t)
    = \frac{1}{p}\,\bigl\|\mu_\theta(s) - \xi^\dagger\bigr\|_2^2,
\end{align}
whose gradient $\nabla_{\mu_\theta}\ell = \frac{2}{p}(\mu_\theta(s) - \xi^\dagger)$ is bounded away from zero whenever $\mu_\theta(s) \neq \xi^\dagger$, regardless of the logit magnitude. This loss serves as a surrogate for $D_{\mathrm{KL}}(\piE^{(t_{cpu})} \| \piT)$: since $\piE^{(t_{cpu})}$ is deterministic, the forward KL reduces to $-\log \piT(u^\dagger \mid s)$, which for a Gaussian actor in logit space is equivalent to~\eqref{eq:ImitationLoss} up to variance terms.

\subsection{Advantage-Based Sigmoid Gate}\label{sec:AdvantageGate}

Applying~\eqref{eq:ImitationLoss} uniformly risks preventing $\piT$ from surpassing $\piE^{(t_{cpu})}$, particularly in the informational asymmetry gap $\mathcal{G}$ (Definition~\ref{def:InfoGap}). Conversely, entirely disabling guidance as in prior study~\cite{beikmohammadi2024accelerating} discards critical signals in which the planner agent remains superior, forcing the agent to revert to standard RL within a restricted observation space that fails to form a valid MDP. This typically leads to catastrophic forgetting and a return to the underlying POMDP's state aliasing. P2P-SAC resolves this via a value-based gate that selectively maintains guidance in aliased states where the planner agent's privileged advantage persists, using the estimated soft state value and planner agent advantage:
\begin{align}
    \widehat{V}(s)
    &= \min_{j \in \{1,2\}} Q_{\phi_j}(s, \tilde{u})
    - \alpha\,\log\piT(\tilde{u} \mid s),
    \label{eq:SoftValue}\\
    \widehat{A}^\dagger(s, u^\dagger)
    &= \min_{j \in \{1,2\}} Q_{\phi_j}(s, u^\dagger)
    - \widehat{V}(s),
    \label{eq:PlannerAdvantage}
\end{align}
where $\tilde{u} \sim \piT(\cdot \mid s)$, and $Q_{\phi_j}$ is the learned critic function conditioned on observations. The advantage gate maps $\widehat{A}^\dagger(s)$ to a soft weight via sigmoid with temperature $\tau_g > 0$:
\begin{align}\label{eq:SigmoidGate}
    m_\phi(s, u^\dagger) = \sigma\!\left(\frac{\widehat{A}^\dagger(s, u^\dagger)}{\tau_g}\right)
    = \frac{1}{1 + \exp\!\left(-\widehat{A}^\dagger(s, u^\dagger)/\tau_g\right)}.
\end{align}
Combining with the maturity indicator yields the composite gating function:
\begin{align}\label{eq:CompositeGate}
    G_\phi(s, u^\dagger;\, M_t)
    = (1 - M_t) + M_t \cdot m_\phi(s, u^\dagger).
\end{align}
In the immature regime ($M_t = 0$), $G_\phi \equiv 1$, applying the anchor uniformly since $Q_{\phi_j}$ is unreliable. In the mature regime ($M_t = 1$), $G_\phi = m_\phi(s, u^\dagger)$: the anchor is suppressed where $\piT$ dominates ($\widehat{A}^\dagger < 0$) and retained where the planner agent is superior. In states $s \in \mathcal{G}$, the gate converges toward $0.5$, asymptotically removing the imitation bias where it is least justified.

\subsection{Composite Actor Objective}\label{sec:ActorObjective}

The actor loss combines the SAC objective~\eqref{eq:SACactor} with the gated anchor:
\begin{align}
    L_{\pi}(\theta) &= L_{\mathrm{SAC}}(\theta) + L_{\mathrm{anchor}}(\theta), \label{eq:TotalActorLoss}
\end{align}
\begin{align}
    L_{\mathrm{SAC}}(\theta)
    &= \mathbb{E}_{\substack{s \sim \mathcal{B},\,
                            \tilde{u} \sim \piT}}
    \!\left[
        \alpha\log\piT(\tilde{u} \mid s)
        - \min_{j} Q_{\phi_j}(s,\tilde{u})
    \right], \label{eq:SACLoss}\\
    L_{\mathrm{anchor}}(\theta)
    &= \mathbb{E}_{(s, u^\dagger, h)\sim \mathcal{B}}
    \!\left[
        \beta_t\,
        G_\phi(s, u^\dagger; M_t)\,
        h
    \right], \label{eq:AnchorLoss}
\end{align}
with $h \in \{0,1\}$ the planner-availability indicator (Definition~\ref{def:DualBuffer}). The product $\beta_t \cdot G_\phi(s, u^\dagger; M_t)$ is the effective guidance weight, encoding both the global training phase and local planner agent superiority. The critic functions $\phi_j$ are frozen during the actor update. The entropy temperature $\alpha$ is updated by minimizing $L_\alpha = \mathbb{E}_{\tilde{u}\sim\piT}[-\alpha(\log\piT(\tilde{u} \mid s) + \bar{\mathcal{H}})]$, and target networks are updated via Polyak averaging: $\phi_{j,\mathrm{targ}} \leftarrow \rho_{\mathrm{poly}}\,\phi_{j,\mathrm{targ}} + (1-\rho_{\mathrm{poly}})\,\phi_j$.

\subsection{Algorithm}\label{sec:AlgoDiscussion}
The P2P-SAC procedure is implemented in Algorithm~\ref{alg:p2psac}. The process begins (\textbf{Lines 3–5}) by querying both the planner agent and actor of the learning agent, selecting an action based on the maturity indicator ($M_t$), and storing the transition in the dual replay buffer. Next (\textbf{Line 6–11}), it samples a mixed batch of data to train the critic function networks. Following this (\textbf{Lines 12–13}), the algorithm computes the advantage gate ($G_\phi$) using a stop-gradient to evaluate the planner agent's usefulness without biasing the critic function's estimation. Finally (\textbf{Lines 14–16}), the actor ($\piT$) is updated using the gated loss, followed by standard updates to the temperature and target networks.

\begin{algorithm}[t]
\caption{P2P-SAC}
\label{alg:p2psac}
\begin{algorithmic}[1]
\Require $\piE^{(t_{cpu})}$; $(T_p, T_d, \beta_0, \beta_f)$; $\tau_g$; $(C_P, C)$; $\rho_{\mathrm{poly}}$; $\bar{\mathcal{H}}$; $B$
\State Initialize $\piT(\theta)$, $Q_{\phi_j}$, $Q_{\phi_{j,\mathrm{targ}}} \!\leftarrow\! Q_{\phi_j}$ for $j \!\in\! \{1,2\}$, $\alpha$, $\mathcal{D}\!=\!(\mathcal{D}_P, \mathcal{D}_O)$, $M_0\!\leftarrow\!0$
\For{$t = 0, 1, 2, \ldots$} 
    \State Collect ($s_t$, $u^\dagger_t$, $\tilde{u}_t$) and generate $u_t$ using ~\eqref{eq:BehaviorSubstitution}
    \State Execute $u_t$; observe $r_t$, $s_{t+1}$
    \State By Definition~\ref{def:Schedule} store $(s_t, u_t, r_t, s_{t+1}, u^\dagger_t)$ in $\mathcal{D}$ and update $\beta_t$, $M_t$ 
    \State Sample $\mathcal{B}$ as~\eqref{eq:MixedSampling}
    \For{$j \in \{1,2\}$} 
        \State $\tilde{u}' \sim \piT(\cdot \mid s')$;\; 
        \State $y \leftarrow r + \gamma\bigl(\min_{j'} Q_{\phi_{j'}}(s',\tilde{u}') - \alpha \log\piT(\tilde{u}'\mid s')\bigr)$
        \State $\phi_j \leftarrow \phi_j - \eta_\phi \,\nabla_{\phi_j}\tfrac{1}{|\mathcal{B}|}\!\sum_{\mathcal{B}}\!(Q_{\phi_j}(s,u) - y)^2$
    \EndFor
    \State Compute $\xi^\dagger$ via~\eqref{eq:PlannerLogit};\; $\ell(s,u^\dagger)$ via~\eqref{eq:ImitationLoss} 
    \State Compute $\widehat{V}$, $\widehat{A}^\dagger$, $m_\phi$, $G_\phi$ via~\eqref{eq:SoftValue}--\eqref{eq:CompositeGate}
    \State $\theta \leftarrow \theta - \eta_\theta\,\nabla_\theta\tfrac{1}{|\mathcal{B}|}\!\sum_{\mathcal{B}}\! L_\pi(\theta)$ \hfill\Comment{Eq.~\eqref{eq:TotalActorLoss}}
    \State $\alpha \leftarrow \alpha - \eta_\alpha\,\nabla_\alpha\tfrac{1}{|\mathcal{B}|}\!\sum_{\mathcal{B}}\!\bigl[-\alpha(\log\piT(\tilde{u}\mid s) + \bar{\mathcal{H}})\bigr]$
    \State $\phi_{j,\mathrm{targ}} \leftarrow \rho_{\mathrm{poly}}\,\phi_{j,\mathrm{targ}} + (1-\rho_{\mathrm{poly}})\,\phi_j$, \;$j \in \{1,2\}$
\EndFor
\end{algorithmic}
\end{algorithm}

\section{Framework Instantiation}\label{sec:instantiation}
This section establishes that (i)~the REAP-based planner agent discussed in Section \ref{sec:AMPC} satisfies
Definition~\ref{def:AnytimePlanner}, and (ii)~P2P-SAC optimizes a
planner-regularized reactive objective whose gradient is immune to the irreducible variance caused by state aliasing.

\begin{corollary}[REAP-Based Planner Agent]\label{lem:reap_valid}
By Propositions~\ref{theorem:Convergence} and~\ref{theorem:Feasibility}, the dynamical system defined in \eqref{eq:REAP} initialized at a feasible
$(\hat{\mathbf{u}}_0, \hat{\lambda}_0)$ satisfies
Definition~\ref{def:AnytimePlanner}.
\end{corollary}

We now characterize the actor gradient of P2P-SAC.
Let $\rho_{\mathcal{D}}(s)$ be the empirical marginal of
observations in $\mathcal{D}$ and
$\nu_{\mathcal{D}}(\xi^\dagger, u^\dagger \mid s)$ the empirical conditional of planner agent's logits and actions given~$s$.  Define the
buffer statistics (under stop-gradient):
\begin{align}
    \bar{m}_{\mathcal{D}}(s)
    &:= \E_{\nu_{\mathcal{D}}}\!\big[m_\phi(s, u^\dagger)\big],
    \label{eq:MeanGate}\\
    \tilde{\xi}_{\mathcal{D}}(s)
    &:= \E_{\nu_{\mathcal{D}}}\!\big[
        m_\phi(s, u^\dagger)\,\xi^\dagger\big]
        / \bar{m}_{\mathcal{D}}(s),
    \label{eq:WeightedLogit}\\
    \widetilde{\mathcal{V}}_{\mathcal{D}}(s)
    &:= \tfrac{1}{p}\!\Big(
        \E_{\nu_{\mathcal{D}}}\!\big[
            m_\phi(s, u^\dagger)\|\xi^\dagger\|^2\big]
        - \bar{m}_{\mathcal{D}}(s)\,
          \|\tilde{\xi}_{\mathcal{D}}(s)\|^2
    \Big),
    \label{eq:AliasingVar}
\end{align}
with the convention
$\tilde{\xi}_{\mathcal{D}}(s)\!=\!\mathbf{0}$,
$\widetilde{\mathcal{V}}_{\mathcal{D}}(s)\!=\!0$
when $\bar{m}_{\mathcal{D}}(s)\!=\!0$.

\begin{theorem}[Planner-Regularized Objective]\label{thm:P2P_objective}
In the mature phase ($M_t\!=\!1$, $h\!=\!1$),
\begin{align}\label{eq:GradEquiv}
    \nabla_\theta L_\pi(\theta)
    = \nabla_\theta L_{\mathrm{SAC}}(\theta)
    + \nabla_\theta R_{\mathrm{P2P}}(\theta),
\end{align}
where the planner agent's regularizer is
\begin{align}\label{eq:RegTerm}
    R_{\mathrm{P2P}}(\theta)
    := \tfrac{\beta_f}{p}\,
    \E_{s \sim \rho_{\mathcal{D}}}\!\big[
        \bar{m}_{\mathcal{D}}(s)\,
        \|\mu_\theta(s) - \tilde{\xi}_{\mathcal{D}}(s)\|^2
    \big].
\end{align}
The gate-weighted aliasing variance
$C = \beta_f\,\E_{s}[\widetilde{\mathcal{V}}_{\mathcal{D}}(s)]$
enters $L_\pi$ but is $\theta$-independent and absent
from~\eqref{eq:GradEquiv}.
\end{theorem}

\begin{proof}
With $M_t\!=\!1$, $h\!=\!1$:
$L_{\mathrm{anchor}}(\theta)
= \tfrac{\beta_f}{p}\,
  \E_{\mathcal{D}}[m_\phi(s,u^\dagger)
  \|\mu_\theta(s)-\xi^\dagger\|^2]$.
Conditioning on $s$ and noting that $\mu_\theta(s)$ is constant
over $\nu_{\mathcal{D}}(\cdot\mid s)$, the weighted
bias--variance identity\footnote{$\E[w\|a-X\|^2]
= \bar{w}\|a-\tilde{X}\|^2
+ \E[w\|X\|^2] - \bar{w}\|\tilde{X}\|^2$
with $\bar{w}=\E[w]$, $\tilde{X}=\E[wX]/\bar{w}$.
Proof: expand $\|a-X\|^2$, substitute $\E[wX]=\bar{w}\tilde{X}$.}
with $a=\mu_\theta(s)$, $X=\xi^\dagger$,
$w=m_\phi(s,u^\dagger)$ gives
$L_{\mathrm{anchor}} = R_{\mathrm{P2P}}(\theta) + C$.
All quantities in $C$ are computed under stop-gradient and
independent of $\theta$, so
$\nabla_\theta L_{\mathrm{anchor}}
= \nabla_\theta R_{\mathrm{P2P}}$.
\end{proof}

\begin{remark}\label{rem:Interpretation}
Three consequences follow from Theorem~\ref{thm:P2P_objective}.
\emph{(a)~Privileged-information distillation:}
$R_{\mathrm{P2P}}$ pulls $\mu_\theta(s)$ toward
$\tilde{\xi}_{\mathcal{D}}(s)$, the gate-weighted average of
planner agent's logits across latent states that produced $s$ in the
buffer, injecting privileged information into a reactive policy.
\emph{(b)~Aliasing-immune gradient:}
the variance $\widetilde{\mathcal{V}}_{\mathcal{D}}(s)$, which
captures the irreducible ambiguity in aliased states
$s\!\in\!\mathcal{G}$, does not enter the policy gradient.
\emph{(c)~Bounded regularization cost:}
$\bar{m}_{\mathcal{D}}(s) \leq 1$ implies
$R_{\mathrm{P2P}}(\theta)
\leq \tfrac{\beta_f}{p}\sup_s
\|\mu_{\theta}(s) - \tilde{\xi}_{\mathcal{D}}(s)\|^2$
for any $\theta$, bounding the maximum penalty the regularizer can impose.
\end{remark}



\section{Simulation and Experimental Evaluation}\label{sec:sim_results}

We evaluate the framework on autonomous quadrupedal navigation, specifying all abstract quantities from Section~\ref{sec:ProblemStatement}.

\subsection{Platform and Observation Instantiation}

The platform is a Unitree Go2 quadruped. Following~\cite{lee2020learning,margolis2022rapid}, a frozen locomotion policy $\pi_{\mathrm{ll}}$~\cite{schwarke2025rsl} converts velocity commands to torques at 200\,Hz, while $\piT$ outputs $u_t = [v_x, v_y]^\top \in \mathfrak{U}$ at 50\,Hz. The observation maps are
\begin{align}
    s_t &= h_s(x_t) = [x_t^{\mathrm{rob}},\, y_t^{\mathrm{rob}},\, x^{\mathrm{goal}},\, y^{\mathrm{goal}}]^\top \in \mathbb{R}^4, \label{eq:sObs}\\
    z_t &= h_z(x_t) = [x_t^{\mathrm{rob}},\, y_t^{\mathrm{rob}}]^\top \in \mathbb{R}^2, \label{eq:zObs}
\end{align}
with goal $(x^{\mathrm{goal}}, y^{\mathrm{goal}}) = (0.0, 2.8)$\,m. Obstacle positions, heading, and joint quantities are excluded from $s_t$ (blind navigation~\cite{siekmann2021blind}), inducing informational incompleteness ($n_s < n$). The planner agent receives the privileged information $\mathcal{I}_t = \{(o_i, r_i)_{i=1}^{6}, (b_i)_{i=1}^{4}, x^{\mathrm{goal}}, y^{\mathrm{goal}}\}$ encoding obstacle and boundary geometry, never communicated to $\piT$, confirming $\mathcal{G} \neq \emptyset$. The planner agent's world-frame velocity is mapped to the body frame via $u_t^\dagger = [u_{y,w}^\dagger,\, -u_{x,w}^\dagger]^\top$ with $u_{low}=u_{high} = 0.5$ m/s. The linear model is a 2D single-integrator at 50\,Hz: $z_{k+1} = z_k + u_k \cdot 0.02$, and $\tilde{\mathcal{Z}}$ defined by linearized obstacle-avoidance halfspaces. By Corollary~\ref{lem:reap_valid}, REAP-based formulation in \eqref{eq:REAP} with $N=15$ and $\beta = 100$ satisfies Definition~\ref{def:AnytimePlanner}. Since obstacle positions and heading are excluded from $s_t$, the learning agent faces a POMDP (Section~\ref{sec:MDP}): multiple latent configurations $(x_t, \mathcal{I}_t)$ project to the same $s_t$, confirming $\mathcal{G} \neq \emptyset$. 

\subsection{Simulation Setup}
Training and evaluation use NVIDIA Isaac Lab~\cite{mittal2025isaac} with the
\texttt{Isaac-Velocity-Flat-Unitree-Go2-v0} task at $\Delta t_{\mathrm{ctrl}} = 0.02$\,s (50\,Hz).
The arena is $4.1 \times 5.6$\,m$^2$ ($x \in [-2.2, 2.0]$, $y \in [-2.0, 3.5]$\,m) with six
cylindrical obstacles of radius $0.23$\,m, arranged symmetrically: one at the entry
$(0.00, 0.15)$\,m, one at the centre $(0.00, 1.45)$\,m, and four flanking obstacles at
$(\pm1.30, 0.75)$\,m and $(\pm1.30, -0.45)$\,m. The same geometry is used identically in Isaac Lab and the REAP-based planner agent.
The robot spawns randomly in the lower half via rejection
sampling~\cite{gilks1992adaptive}. Episodes terminate on goal success ($< 0.3$\,m),
collision, fall (trunk $< 0.1$\,m), or timeout ($T_{\max} = 8{,}000$ steps). Five seeds
$\{0,\ldots,4\}$ per algorithm are run on the NVIDIA A40 GPU. To make the problem challenging for the algorithms, a sparse reward is defined as $r(u_t) = - c_{\mathrm{step}} + r_{\mathrm{mag}}(u_t)$, where $c_{\mathrm{step}} = 1.0$,
$r_{\mathrm{mag}} = -0.02\|u_t\|_2^2$, with terminal rewards $+100$ (goal) and
$-200$ (crash).

\subsection{Compared Algorithms and Hyperparameters}

\noindent\textbf{SAC}~\cite{haarnoja2018soft}: vanilla maximum-entropy actor--critic, without planner agent.
\noindent\textbf{PPO}~\cite{schulman2017proximal}: Standard on-policy
policy gradient with clip ratio $\epsilon = 0.2$, without planner agent.
\textbf{Accelerated SAC}~\cite{beikmohammadi2024accelerating}: output-space pseudo-label loss with plateau-then-decay schedule ($T_p = 10^5$, $T_d = 5 \times 10^4$, $\beta_0 = 10.0$); single buffer.
\textbf{P2P-SAC}: Algorithm~\ref{alg:p2psac} with $\beta_0 = \beta_f = 10.0$, $T_p = 10^5$, $T_d = 0$, $\tau_g = 1.0$, $C_P = 10^6$, $C = 2 \times 10^6$; REAP-based planner agent with $N = 15$, $\beta = 100$; agent'a action bounds $[-0.7, 0.7]^2$. All methods share the same architecture: two hidden layers of 256 units (ReLU), Adam with $\mathrm{lr} = 3 \times 10^{-4}$. Note that $T_d = 0$ collapses the annealing phase; the sole change at $t = T_p$ is activation of $m_\phi(s)$, isolating the gate's contribution.

\subsection{Evaluation Metrics}

Table.~\ref{tab:results} summarizes the evaluation metrics are computed over the last 10 episodes with different seeds on the trained policies: success rate, crash rate, path optimality $\ell_{\mathrm{ep}} / \|p^{\mathrm{goal}} - p^{\mathrm{spawn}}\|_2$, runtime, and average velocity. 

\subsection{Results and Discussion}

\subsubsection{Sample efficiency}
As it is shown in Fig.~\ref{fig:training_curves}, in the training, P2P-SAC achieves 100\% success after 1M steps, versus 40\% for Accelerated SAC. The vanilla SAC and PPO fail at this task because they operate in the POMDP. 

\subsubsection{Final performance}
The improvement of P2P-SAC over Accelerated SAC is attributable to two factors: the logit-space anchor provides non-vanishing gradients near $\partial\mathfrak{U}$, and the advantage gate preserves the imitation loss, and selectively suppresses imitation in $\mathcal{G}$ where the planner agent's privileged $\mathcal{I}_t$ confers an unreplicable advantage. In P2P-SAC, setting $T_p = 10^5$ enables the planner to collect high-quality trajectories right from the start of training. This immediate proficiency results in a 100\% success rate, as illustrated in Fig.~\ref{fig:training_curves}, and empirically demonstrates the anytime feasibility of the REAP \eqref{eq:REAP}.

\subsubsection{Advantage gate behaviour}
During the annealing phase, $G_\phi \equiv 1$ by~\eqref{eq:CompositeGate}. At maturation ($t = T_p$), $G_\phi$ drops to $\approx 0.1$ as the critic function initially estimates $\piT$ as superior, then stabilizes at $G_\phi \approx 0.45$, consistent with the prediction that $m_\phi(s, u^\dagger) \to 0.5$ in $\mathcal{G}$.

\subsubsection{Path quality}
The dual buffer ensures the critic function bootstraps from the planner agent's trajectories, yielding path optimality of $1.06$ versus $1.10$ for REAP.

\begin{table}[t]
    \centering
    \caption{Best-checkpoint metrics (mean $\pm$ std, 5 seeds).}
    \label{tab:results}
    \resizebox{\columnwidth}{!}{%
    \begin{tabular}{lccccc}
        \hline
        Algorithm & Success (\%) & Crash (\%) & Path opt.\ & Runtime (s) \ & Ave. Velocity (m/s) \\
        \hline
        Accel.\ SAC~\cite{beikmohammadi2024accelerating}
        & 35.0 $\pm$ 47.7& 65.0 $\pm$ 47.7 & 1.100 $\pm$ 0.073 & 9.0 $\pm$ 1.3 & 0.477 $\pm$ 0.045\\
        {P2P-SAC}
        & {100\%} & {0.0\%} & {1.060 $\pm$ 0.031} & {9.7 $\pm$ 1.1} & {0.352 $\pm$ 0.019}\\
        \hline
        REAP~\cite{Hosseinzadeh2023RobustTermination}
        & 100\% & 0.0\% & 1.10 $\pm$ 0.04 & 12.26 $\pm$ 1.29 & 0.353 $\pm$ 0.028\\
         \hline
        \hline
    \end{tabular}}
\end{table}

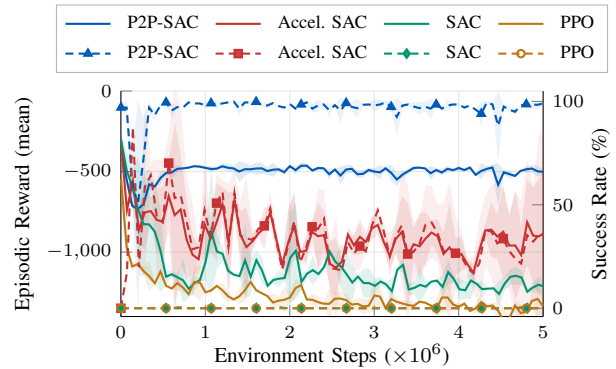
\begin{figure}[t]
    \centering
    \begin{tikzpicture}

    \begin{axis}[
        name=mainax,
        width=0.65\columnwidth,
        height=3cm,
        scale only axis,
        xlabel={Environment Steps ($\times 10^6$)},
        ylabel={Episodic Reward (mean)},
        xmin=0, xmax=5,
        ymin=-1400, ymax=0,
        axis y line*=left,
        axis x line*=bottom,
        ylabel style={yshift=-2pt},
        every axis x label/.style={at={(0.5,-0.08)}, anchor=north},
        legend style={
            at={(0.5,1.08)},
            anchor=south,
            font=\scriptsize,
            draw=gray!60,
            fill=white,
            fill opacity=0.9,
            text opacity=1,
            row sep=1pt,
            column sep=4pt,
            legend columns=4,
        },
        tick label style={font=\scriptsize},
        label style={font=\footnotesize},
        grid=major,
        grid style={gray!20},
        clip=true,
    ]

    \addplot[name path=p2p_rhi, draw=none, forget plot]
        table[x=steps_M, y=rew_hi, col sep=comma, row sep=newline]{results/data_p2psac.csv};
    \addplot[name path=p2p_rlo, draw=none, forget plot]
        table[x=steps_M, y=rew_lo, col sep=comma, row sep=newline]{results/data_p2psac.csv};
    \addplot[fill=cP2P, fill opacity=0.12, forget plot]
        fill between[of=p2p_rhi and p2p_rlo];
    \addplot[cP2P, thick, solid]
        table[x=steps_M, y=rew_mean, col sep=comma, row sep=newline]{results/data_p2psac.csv};
    \addlegendentry{P2P-SAC}

    \addplot[name path=ac_rhi, draw=none, forget plot]
        table[x=steps_M, y=rew_hi, col sep=comma, row sep=newline]{results/data_accelsac.csv};
    \addplot[name path=ac_rlo, draw=none, forget plot]
        table[x=steps_M, y=rew_lo, col sep=comma, row sep=newline]{results/data_accelsac.csv};
    \addplot[fill=cAccel, fill opacity=0.12, forget plot]
        fill between[of=ac_rhi and ac_rlo];
    \addplot[cAccel, thick, solid]
        table[x=steps_M, y=rew_mean, col sep=comma, row sep=newline]{results/data_accelsac.csv};
    \addlegendentry{Accel.\ SAC}

    \addplot[name path=sac_rhi, draw=none, forget plot]
        table[x=steps_M, y=rew_hi, col sep=comma, row sep=newline]{results/data_sac.csv};
    \addplot[name path=sac_rlo, draw=none, forget plot]
        table[x=steps_M, y=rew_lo, col sep=comma, row sep=newline]{results/data_sac.csv};
    \addplot[fill=cSAC, fill opacity=0.12, forget plot]
        fill between[of=sac_rhi and sac_rlo];
    \addplot[cSAC, thick, solid]
        table[x=steps_M, y=rew_mean, col sep=comma, row sep=newline]{results/data_sac.csv};
    \addlegendentry{SAC}

    \addplot[name path=ppo_rhi, draw=none, forget plot]
        table[x=steps_M, y=rew_hi, col sep=comma, row sep=newline]{results/data_ppo.csv};
    \addplot[name path=ppo_rlo, draw=none, forget plot]
        table[x=steps_M, y=rew_lo, col sep=comma, row sep=newline]{results/data_ppo.csv};
    \addplot[fill=cPPO, fill opacity=0.12, forget plot]
        fill between[of=ppo_rhi and ppo_rlo];
    \addplot[cPPO, thick, solid]
        table[x=steps_M, y=rew_mean, col sep=comma, row sep=newline]{results/data_ppo.csv};
    \addlegendentry{PPO}

    \addlegendimage{cP2P, thick, densely dashed, mark=triangle*, mark size=1.6pt,
        mark options={solid, fill=cP2P}}
    \addlegendentry{P2P-SAC}

    \addlegendimage{cAccel, thick, densely dashed, mark=square*, mark size=1.4pt,
        mark options={solid, fill=cAccel}}
    \addlegendentry{Accel.\ SAC}

    \addlegendimage{cSAC, thick, densely dashed, mark=diamond*, mark size=1.6pt,
        mark options={solid, fill=cSAC}}
    \addlegendentry{SAC}

    \addlegendimage{cPPO, thick, densely dashed, mark=o, mark size=1.4pt,
        mark options={solid, draw=cPPO, fill=white}}
    \addlegendentry{PPO}

    \end{axis}

    \begin{axis}[
        width=0.65\columnwidth,
        height=3cm,
        scale only axis,
        xmin=0, xmax=5,
        ymin=-4, ymax=105,
        axis y line*=right,
        axis x line=none,
        ylabel={Success Rate (\%)},
        ylabel style={yshift=4pt},
        tick label style={font=\scriptsize},
        label style={font=\footnotesize},
        clip=true,
    ]

    \addplot[name path=p2p_shi, draw=none, forget plot]
        table[x=steps_M, y=suc_hi, col sep=comma, row sep=newline]{results/data_p2psac.csv};
    \addplot[name path=p2p_slo, draw=none, forget plot]
        table[x=steps_M, y=suc_lo, col sep=comma, row sep=newline]{results/data_p2psac.csv};
    \addplot[fill=cP2P, fill opacity=0.08, forget plot]
        fill between[of=p2p_shi and p2p_slo];
    \addplot[cP2P, thick, densely dashed, mark=triangle*, mark size=1.6pt,
        mark repeat=8, mark options={solid, fill=cP2P}]
        table[x=steps_M, y=suc_mean, col sep=comma, row sep=newline]{results/data_p2psac.csv};

    \addplot[name path=ac_shi, draw=none, forget plot]
        table[x=steps_M, y=suc_hi, col sep=comma, row sep=newline]{results/data_accelsac.csv};
    \addplot[name path=ac_slo, draw=none, forget plot]
        table[x=steps_M, y=suc_lo, col sep=comma, row sep=newline]{results/data_accelsac.csv};
    \addplot[fill=cAccel, fill opacity=0.08, forget plot]
        fill between[of=ac_shi and ac_slo];
    \addplot[cAccel, thick, densely dashed, mark=square*, mark size=1.4pt,
        mark repeat=8, mark options={solid, fill=cAccel}]
        table[x=steps_M, y=suc_mean, col sep=comma, row sep=newline]{results/data_accelsac.csv};

    \addplot[cSAC, thick, densely dashed, mark=diamond*, mark size=1.6pt,
        mark repeat=8, mark options={solid, fill=cSAC}]
        table[x=steps_M, y=suc_mean, col sep=comma, row sep=newline]{results/data_sac.csv};

    \addplot[cPPO, thick, densely dashed, mark=o, mark size=1.4pt,
        mark repeat=8, mark options={solid, draw=cPPO, fill=white}]
        table[x=steps_M, y=suc_mean, col sep=comma, row sep=newline]{results/data_ppo.csv};

    \end{axis}

    \end{tikzpicture}
    \caption{Training curves: mean episodic reward (solid lines, left y-axis) and success rate (dashed lines with markers, right y-axis) over environment steps; shaded regions indicate $\pm 1$ standard deviation across seeds.}
    \label{fig:training_curves}
    \vspace{-0.5cm}
\end{figure}

\subsection{Real-World Evaluation}\label{sec:ExperimentalValidation}

The framework is validated on a physical Unitree Go2 quadruped. A remote unit (Intel i9-13900K, 64\,GB RAM) executes the planning algorithms, communicating via Wi-Fi. State estimation is provided by an OptiTrack system (ten Prime$^{\text{x}}$\,13 cameras, 120\,Hz, $\pm 0.02$\,mm accuracy). The closed-loop control operates at 50\,Hz. A video demonstration of the hardware deployment, along with the complete source code, is available at GitHub.\footnote{\scriptsize \url{https://github.com/mohsen1amiri/PriPG-RL_UnitreeGo2.git}}\vspace{-0.1cm}

\begin{figure}[h]
    \centering
    \includegraphics[width=0.95\linewidth]{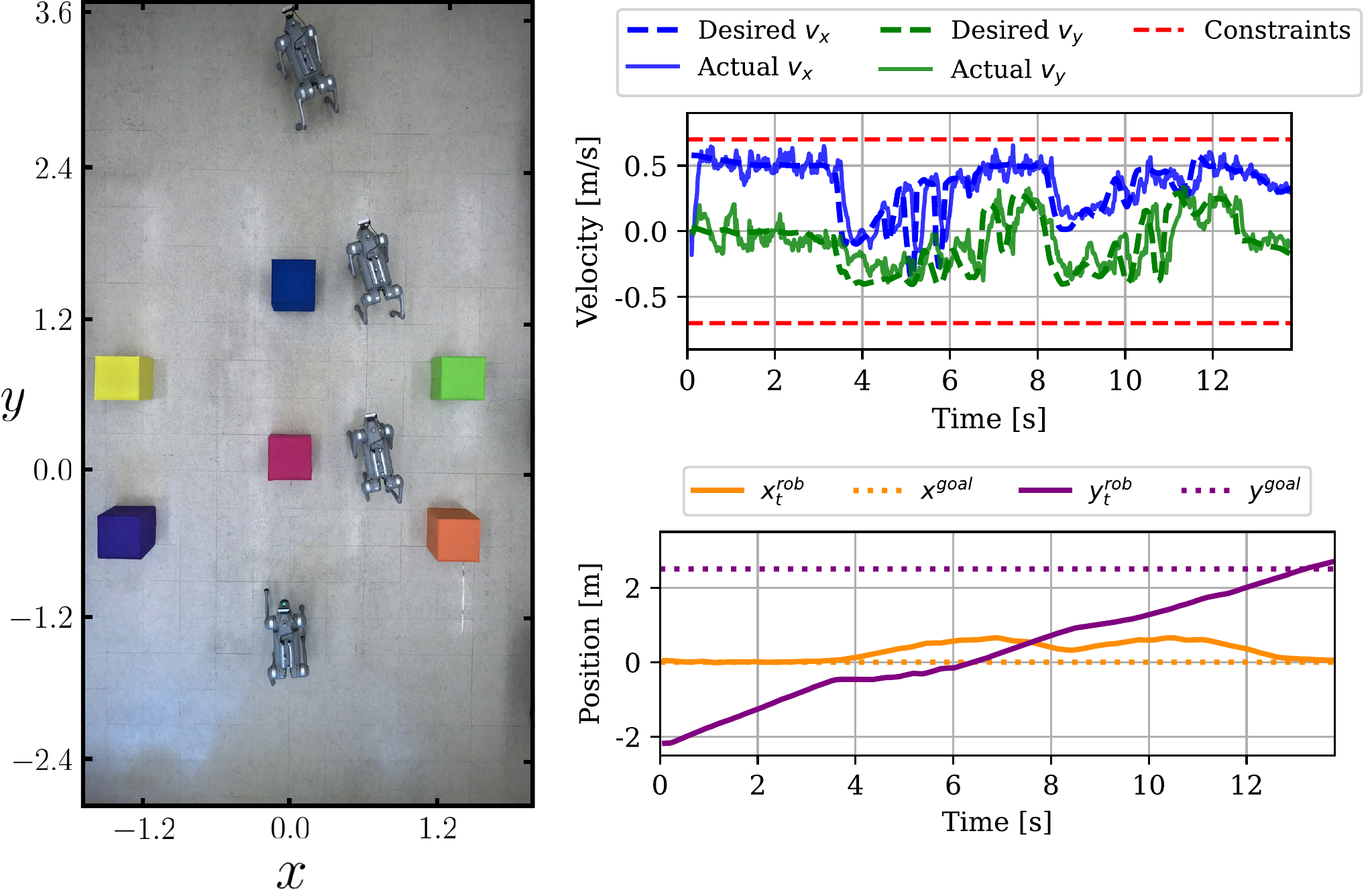}
    \vspace{-.5em}\caption{Hardware experiments using the Unitree Go2 quadruped. The left subfigure shows a composite image of the Unitree Go2 navigating an obstacle-rich environment under P2P-SAC. The top-right subfigure illustrates the desired velocities generated by the proposed method and the actual velocities measured by the onboard hardware, demonstrating that the velocity constraints are satisfied at all times. The bottom-right subfigure presents the trajectory in Cartesian coordinates along with the goal X-Y position.}\vspace{-.5em}
    \label{fig:experiment}
\end{figure}

Fig.~\ref{fig:experiment} shows the experimental results. The quadruped successfully avoids all obstacles within the velocity constraints and reaches the goal, demonstrating that the policy trained via P2P-SAC transfers to hardware and maintains safe trajectories under real-world conditions.



\section{Conclusion}\label{sec:conclusion}

We presented PriPG-RL, a framework for training reactive RL policies
under partial observability by leveraging an anytime-feasible planner agent, which is available only during training. The framework pairs two instantiations: REAP as an anytime-feasible MPC planner agent, and P2P-SAC as a learning agent whose planner-regularized objective provably separates useful privileged guidance from
irreducible aliasing variance
(Theorem~\ref{thm:P2P_objective}). Simulation in NVIDIA Isaac Lab
and deployment on a Unitree Go2 quadruped confirm that P2P-SAC
achieves reliable obstacle avoidance in a POMDP setting where
standard SAC and PPO fail entirely. Future work will extend the PriPG-RL framework beyond reactive policies by proposing a time-varying, anytime-feasible planner agent to supervise history-aware architectures, thereby resolving the temporal ambiguities introduced by non-stationary environments.

\bibliographystyle{IEEEtran.bst} 
\bibliography{ref}

@book{sutton1998reinforcement,
  title={Reinforcement learning: An introduction},
  author={Sutton, Richard S and Barto, Andrew G and others},
  volume={1},
  number={1},
  year={1998},
  publisher={MIT press Cambridge}
}

@article{amiri2026dynamic,
  title={A dynamic embedding method for the real-time solution of time-varying constrained convex optimization problems},
  author={Amiri, Mohsen and Kolmanovsky, Ilya and Hosseinzadeh, Mehdi},
  journal={Systems \& Control Letters},
  volume={209},
  pages={106352},
  year={2026},
  publisher={Elsevier}
}

@article{mnih2015human,
  title={Human-level control through deep reinforcement learning},
  author={Mnih, Volodymyr and Kavukcuoglu, Koray and Silver, David and Rusu, Andrei A and Veness, Joel and Bellemare, Marc G and Graves, Alex and Riedmiller, Martin and Fidjeland, Andreas K and Ostrovski, Georg and others},
  journal={nature},
  volume={518},
  number={7540},
  pages={529--533},
  year={2015},
  publisher={Nature Publishing Group}
}

@inproceedings{duan2016benchmarking,
  title={Benchmarking deep reinforcement learning for continuous control},
  author={Duan, Yan and Chen, Xi and Houthooft, Rein and Schulman, John and Abbeel, Pieter},
  booktitle={International conference on machine learning},
  pages={1329--1338},
  year={2016},
  organization={PMLR}
}

@inproceedings{beikmohammadi2023ta,
  title={Ta-explore: Teacher-assisted exploration for facilitating fast reinforcement learning},
  author={Beikmohammadi, Ali and Magn{\'u}sson, Sindri},
  booktitle={Proceedings of the 2023 International Conference on Autonomous Agents and Multiagent Systems},
  pages={2412--2414},
  year={2023}
}

@article{janner2019trust,
  title={When to trust your model: Model-based policy optimization},
  author={Janner, Michael and Fu, Justin and Zhang, Marvin and Levine, Sergey},
  journal={Advances in neural information processing systems},
  volume={32},
  year={2019}
}

@article{amodei2016concrete,
  title={Concrete problems in AI safety},
  author={Amodei, Dario and Olah, Chris and Steinhardt, Jacob and Christiano, Paul and Schulman, John and Man{\'e}, Dan},
  journal={arXiv preprint arXiv:1606.06565},
  year={2016}
}

@article{uesato2018rigorous,
  title={Rigorous agent evaluation: An adversarial approach to uncover catastrophic failures},
  author={Uesato, Jonathan and Kumar, Ananya and Szepesvari, Csaba and Erez, Tom and Ruderman, Avraham and Anderson, Keith and Heess, Nicolas and Kohli, Pushmeet and others},
  journal={arXiv preprint arXiv:1812.01647},
  year={2018}
}

@inproceedings{haarnoja2018soft,
  title={Soft actor-critic: Off-policy maximum entropy deep reinforcement learning with a stochastic actor},
  author={Haarnoja, Tuomas and Zhou, Aurick and Abbeel, Pieter and Levine, Sergey},
  booktitle={International conference on machine learning},
  pages={1861--1870},
  year={2018},
  organization={Pmlr}
}

@article{schulman2017proximal,
  title={Proximal policy optimization algorithms},
  author={Schulman, John and Wolski, Filip and Dhariwal, Prafulla and Radford, Alec and Klimov, Oleg},
  journal={arXiv preprint arXiv:1707.06347},
  year={2017}
}

@inproceedings{fujimoto2018addressing,
  title={Addressing function approximation error in actor-critic methods},
  author={Fujimoto, Scott and Hoof, Herke and Meger, David},
  booktitle={International conference on machine learning},
  pages={1587--1596},
  year={2018},
  organization={PMLR}
}

@article{beikmohammadi2024accelerating,
  title={Accelerating actor-critic-based algorithms via pseudo-labels derived from prior knowledge},
  author={Beikmohammadi, Ali and Magn{\'u}sson, Sindri},
  journal={Information Sciences},
  volume={661},
  pages={120182},
  year={2024},
  publisher={Elsevier}
}

@article{Hosseinzadeh2023RobustTermination,
  title={Robust-to-early termination model predictive control},
  author={Hosseinzadeh, Mehdi and Sinopoli, Bruno and Kolmanovsky, Ilya and Baruah, Sanjoy},
  journal={IEEE transactions on automatic control},
  volume={69},
  number={4},
  pages={2507--2513},
  year={2023},
  publisher={IEEE}
}

@article{amiri2025practical,
  title={Practical considerations for implementing robust-to-early termination model predictive control},
  author={Amiri, Mohsen and Hosseinzadeh, Mehdi},
  journal={Systems \& Control Letters},
  volume={196},
  pages={106018},
  year={2025},
  publisher={Elsevier}
}

@article{amiri2025reap,
  title={{REAP-T}: A {MATLAB} Toolbox for Implementing Robust-to-Early Termination Model Predictive Control},
  author={Amiri, Mohsen and Hosseinzadeh, Mehdi},
  journal={IFAC-PapersOnLine},
  volume={59},
  number={30},
  pages={1096--1101},
  year={2025},
  publisher={Elsevier}
}

@article{oh2025discovering,
  title={Discovering state-of-the-art reinforcement learning algorithms},
  author={Oh, Junhyuk and Farquhar, Gregory and Kemaev, Iurii and Calian, Dan A and Hessel, Matteo and Zintgraf, Luisa and Singh, Satinder and Van Hasselt, Hado and Silver, David},
  journal={Nature},
  volume={648},
  number={8093},
  pages={312--319},
  year={2025},
  publisher={Nature Publishing Group UK London}
}

@article{shakya2023reinforcement,
  title={Reinforcement learning algorithms: A brief survey},
  author={Shakya, Ashish Kumar and Pillai, Gopinatha and Chakrabarty, Sohom},
  journal={Expert Systems with Applications},
  volume={231},
  pages={120495},
  year={2023},
  publisher={Elsevier}
}

@article{lee2020learning,
  title={Learning quadrupedal locomotion over challenging terrain},
  author={Lee, Joonho and Hwangbo, Jemin and Wellhausen, Lorenz and Koltun, Vladlen and Hutter, Marco},
  journal={Science robotics},
  volume={5},
  number={47},
  pages={eabc5986},
  year={2020},
  publisher={American Association for the Advancement of Science}
}

@inproceedings{siekmann2021blind,
  title={Blind bipedal stair traversal via sim-to-real reinforcement learning},
  author={Siekmann, Jonah and Godse, Yesh and Fern, Alan and Hurst, Jonathan},
  booktitle={Robotics: Science and Systems},
  year={2021}
}

@inproceedings{margolis2022rapid,
  title={Rapid locomotion via reinforcement learning},
  author={Margolis, Gabriel B and Yang, Ge and Paull, Liam and Agrawal, Pulkit},
  booktitle={Robotics: Science and Systems},
  year={2022}
}

@inproceedings{hester2018deep,
  title={Deep q-learning from demonstrations},
  author={Hester, Todd and Vecerik, Matej and Pietquin, Olivier and Lanctot, Marc and Schaul, Tom and Piot, Bilal and Horgan, Dan and Quan, John and Sendonaris, Andrew and Osband, Ian and others},
  booktitle={Proceedings of the AAAI conference on artificial intelligence},
  volume={32},
  number={1},
  year={2018}
}

@article{vecerik2018leveraging,
  title={Leveraging demonstrations for deep reinforcement learning on robotics problems with sparse rewards},
  author={Vecerik, Mel and Hester, Todd and Scholz, Jonathan and Wang, Fumin and Pietquin, Olivier and Piot, Bilal and Heess, Nicolas and Roth{\"o}rl, Thomas and Lampe, Thomas and Riedmiller, Martin},
  journal={arXiv preprint arXiv:1707.08817},
  year={2017}
}

@inproceedings{nair2018overcoming,
  title={Overcoming exploration in reinforcement learning with demonstrations},
  author={Nair, Ashvin and McGrew, Bob and Andrychowicz, Marcin and Zaremba, Wojciech and Abbeel, Pieter},
  booktitle={2018 IEEE international conference on robotics and automation (ICRA)},
  pages={6292--6299},
  year={2018},
  organization={IEEE}
}

@article{nair2020awac,
  title={Awac: Accelerating online reinforcement learning with offline datasets},
  author={Nair, Ashvin and Gupta, Abhishek and Dalal, Murtaza and Levine, Sergey},
  journal={arXiv preprint arXiv:2006.09359},
  year={2020}
}

@article{mittal2025isaac,
  title={Isaac lab: A gpu-accelerated simulation framework for multi-modal robot learning},
  author={Mittal, Mayank and Roth, Pascal and Tigue, James and Richard, Antoine and Zhang, Octi and Du, Peter and Serrano-Munoz, Antonio and Yao, Xinjie and Zurbr{\"u}gg, Ren{\'e} and Rudin, Nikita and others},
  journal={arXiv preprint arXiv:2511.04831},
  year={2025}
}

@article{schwarke2025rsl,
  title={Rsl-rl: A learning library for robotics research},
  author={Schwarke, Clemens and Mittal, Mayank and Rudin, Nikita and Hoeller, David and Hutter, Marco},
  journal={arXiv preprint arXiv:2509.10771},
  year={2025}
}

@article{gilks1992adaptive,
  title={Adaptive rejection sampling for Gibbs sampling},
  author={Gilks, Walter R and Wild, Pascal},
  journal={Journal of the Royal Statistical Society: Series C (Applied Statistics)},
  volume={41},
  number={2},
  pages={337--348},
  year={1992},
  publisher={Wiley Online Library}
}

@article{lowrey2018plan,
  title={Plan online, learn offline: Efficient learning and exploration via model-based control},
  author={Lowrey, Kendall and Rajeswaran, Aravind and Kakade, Sham and Todorov, Emanuel and Mordatch, Igor},
  journal={arXiv preprint arXiv:1811.01848},
  year={2018}
}

@inproceedings{nagabandi2018neural,
  title={Neural network dynamics for model-based deep reinforcement learning with model-free fine-tuning},
  author={Nagabandi, Anusha and Kahn, Gregory and Fearing, Ronald S and Levine, Sergey},
  booktitle={2018 IEEE international conference on robotics and automation (ICRA)},
  pages={7559--7566},
  year={2018},
  organization={IEEE}
}

@article{singh1994learning,
  title={Learning without state-estimation in partially observable 
         Markovian decision processes},
  author={Singh, Satinder P and Jaakkola, Tommi and Jordan, Michael I},
  journal={ICML},
  year={1994}
}

@article{amiri2025reinforcement,
  title={Reinforcement learning in switching non-stationary markov decision processes: Algorithms and convergence analysis},
  author={Amiri, Mohsen and Magn{\'u}sson, Sindri},
  journal={arXiv preprint arXiv:2503.18607},
  year={2025}
}

@inproceedings{amiri2024convergence,
  title={On the convergence of td-learning on markov reward processes with hidden states},
  author={Amiri, Mohsen and Magn{\'u}sson, Sindri},
  booktitle={2024 European Control Conference (ECC)},
  pages={2097--2104},
  year={2024},
  organization={IEEE}
}

@inproceedings{chen2020learning,
   title={Learning by cheating},
   author={Chen, Dian and Zhou, Brady and Koltun, Vladlen and Kr{\"a}henb{\"u}hl, Philipp},
   booktitle={Proc. Conference on Robot Learning (CoRL)},
   pages={66--75},
   year={2020}
 }

@inproceedings{kumar2021rma,
   title={{RMA}: Rapid motor adaptation for legged robots},
   author={Kumar, Ashish and Fu, Zipeng and Pathak, Deepak and Malik, Jitendra},
   booktitle={Proc. Robotics: Science and Systems (RSS)},
   year={2021}
 }

@article{lauri2022partially,
  title={Partially observable markov decision processes in robotics: A survey},
  author={Lauri, Mikko and Hsu, David and Pajarinen, Joni},
  journal={IEEE Transactions on Robotics},
  volume={39},
  number={1},
  pages={21--40},
  year={2022},
  publisher={IEEE}
}

@article{ren2022tutorial,
  title={A tutorial review of neural network modeling approaches for model predictive control},
  author={Ren, Yi Ming and Alhajeri, Mohammed S and Luo, Junwei and Chen, Scarlett and Abdullah, Fahim and Wu, Zhe and Christofides, Panagiotis D},
  journal={Computers \& Chemical Engineering},
  volume={165},
  pages={107956},
  year={2022},
  publisher={Elsevier}
}

\end{document}